\documentclass[11pt]{article}
\usepackage[hidelinks]{hyperref}
\usepackage{float}
\usepackage[version = 4]{mhchem}
\usepackage{pdfpages}
\usepackage{mathrsfs}
\usepackage{longtable}
\usepackage{listings}
\usepackage{textcomp}
\usepackage{cancel}
\usepackage{enumerate}
\usepackage{amsmath}
\usepackage{amsthm}
\usepackage{amssymb}
\usepackage{booktabs}
\usepackage{colortbl}
\usepackage{graphicx}
\usepackage{subfigure}
\usepackage{siunitx}
\usepackage{appendix}
\usepackage{wrapfig}
\usepackage{geometry}
\usepackage{indentfirst}
\usepackage{blkarray}
\usepackage{multirow} 
\usepackage{amsmath}
\usepackage{amssymb}
\usepackage{makecell}
\usepackage{pifont}
\usepackage{natbib}

\usepackage{JI_MathCourse_Notations}

\newtheorem{defn}{Definition}[section]
\newtheorem{thm}{Theorem}[section]
\newtheorem{lemma}{Lemma}[section]
\newtheorem{prop}{Proposition}[section]

\theoremstyle{definition}
\newtheorem{remark}{Remark}[section]

\newcommand{\critR}{(\nabla R)^{-1}(0)}
\newcommand{\Rzero}{R^{-1}(0)}

\newcommand{\ta}{\Tilde{a}}
\newcommand{\tw}{\Tilde{w}}
\newcommand{\ttheta}{\Tilde{\theta}}

\title{\textbf{Geometry of Critical Sets and Existence of Saddle Branches for Two-layer Neural Networks}}
\date{}

\author{Leyang Zhang\footnote{leyangz\_hawk@outlook.com}, Yaoyu Zhang\textsuperscript{\rm 1}\footnote{zhyy.sjtu@sjtu.edu.cn}, Tao Luo\textsuperscript{\rm 1,2}\footnote{luotao41@sjtu.edu.cn, corresponding author}\\
\textsuperscript{\rm 1} School of Mathematical Sciences, Institute of Natural Sciences, MOE-LSC \\ 
Shanghai Jiao Tong University \\
\textsuperscript{\rm 2} CMA-Shanghai, Shanghai Artificial Intelligence Laboratory\\
}

\begin{document}

\maketitle

\begin{abstract}
    This paper presents a comprehensive analysis of critical point sets in two-layer neural networks. To study such complex entities, we introduce the critical embedding operator and critical reduction operator as our tools. Given a critical point, we use these operators to uncover the whole underlying critical set representing the same output function, which exhibits a hierarchical structure. Furthermore, we prove existence of saddle branches for any critical set whose output function can be represented by a narrower network. Our results provide a solid foundation to the further study of optimization and training behavior of neural networks.
    % When the output function can be represented by a narrower network than the model, we show there exist saddles in the underlying critical set representing this output function.
\end{abstract}

\section{Introduction}

Neural networks have achieved success in a wide range of applications, but their high performance is less understood. Theoretical studies are thus made to uncover such mysteries. One major area in theoretical study is the analysis of loss landscape. The study is challenging because of loss landscapes' various possible shapes \cite{RHong, Skorokhodov}, the high dimensionality problem, and its complicated dependence on data, model structure, and loss function \cite{OCalin, RSunOverview}. \\

In recent years numerous works have been trying to understand the loss landscape of neural networks in various ways by focusing on its simple structures behind these complicated phenomenon. One direction is discovering the general structures and properties of overparameterized models that are possessed by neural networks. Examples include the analysis of global minimum dimension \cite{YCooper}, which only requires smoothness and overparameterized setting, and the optimistic sample estimate, which works for almost any models analytic in parameters and inputs \cite{OptmisticEstimate}. Another direction focuses on the special properties of neural networks, such as the study of critical points. For example, criticality is preserved when increasing model width \cite{EbddPrincipleLong, EbddPrincipleShort}, saddles instead of local minima are what mainly affects training dynamics \cite{Dauphin}, non-existence of spurious valleys for wide neural networks \cite{LVenturi}, etc. However, we still do not have a clear picture about the geometry and functional properties of the set of critical points. \\

In this paper, we make a step further by studying the critical sets, i.e., sets of critical points, of a neural network. Focusing on two-layer neural networks and utilizing the special properties of it, we characterize the geometry of the critical sets representing a given critical function, and we discuss the existence of saddles in these sets. Surprisingly, we show that these sets have a hierarchical but simple structure, and we identify lots of saddles in these sets. More precisely, our contribution in this work can be summarized as follows.
\begin{itemize}
    \item [(a)] Given a critical point of a neural network, the critical set representing the same output function is a finite union of its subsets. Each subset is a Cartesian product of an Euclidean space and several identical submanifolds. 
    \item [(b)] Moreover, if the output function can be represented by a narrower (fewer neurons) network than the model, there exist many saddles (saddle branches) in the critical set representing it. 
    \item [(c)] We present two maps, the critical embedding and critical reduction operators to help us study the geometry and functional properties of critical sets. 
\end{itemize}

\section{Related Works}\label{Section Related Works}

\textbf{Geometry of critical sets.} There have been many works studying the geometry of critical sets. For example, the Embedding Principle shows that the loss landscape of a neural network contains those of narrower ones \citep{KFukumizu, KFukumizu2, BSimsek, EbddPrincipleShort, EbddPrincipleLong}. Using a critical embedding operator which preserves output function and criticality of loss, these works characterize a subset of the critical points representing a given output function. In this paper we follow this idea, but also introduce another criticality-preserving operator, i.e., critical reduction operator. We use both operators to give a complete characterization of these sets. Other works focus on global minima. \cite{YCooper} shows how sample size determines the dimension of global minima for generic samples in overparameterized regime. Under teacher-student setting, ref. \cite{LZhangGlobal} gives a complete characterization of the structure and gradient dynamics near global minima. In contrast, our paper mainly focus on non-global critical points/sets. \\

\textbf{Analysis of Saddles.} Under the belief that saddles rarely trap gradient methods \citep{Skorokhodov}, some works try to show the prevalence of saddles in loss landscape of a neural network. Refs. \cite{KFukumizu, KFukumizu2, BSimsek, EbddPrincipleShort, EbddPrincipleLong} showed that in general embedding a local minimum of a narrower network to a wider one tends to produce strict saddles. Additionally, research by \cite{LVenturi} and \cite{RSunWideNN} revealed that when the width of a neural network exceeds the sample size, saddles not only exist but in fact there are no spurious valleys. Similar works have been made for deep linear network \citep{Nguyen1, Nguyen2}, and furthermore classifies its strict and non-strict saddles \citep{EAchour}. Our work describes two types of saddles, one called the ``embedding saddles'', while the other produces ``saddle branches'' which cannot be described by Embedding Principle. 

% \textbf{Loss landscape of wide neural networks.} Lots of previous works focus on wide networks. For example, as mentioned above, Ref. \cite{LVenturi} and \cite{RSunWideNN} consider neural networks whose width exceeds the sample size. In this paper we do not pose any requirement on the exact width of our model neural network. Instead, we discuss the relation between the minimal output function width and model width. Thus, the results in our paper are quite general and work for models with any widths.  

\section{Main Results}

\subsection{Preliminaries}\label{Subsection Preliminaries}

Throughout the paper we will work with two-layer (fully-connected) neural networks. Given $m \in \bN$, a two-layer neural network of width $m$ is denoted by $g_m$, or just $g$ when the width is clear from context: 
\begin{equation}
    g(\theta, x) = g_m(\theta, x) = \sum_{k=1}^m a_k, \sigma(w_k^\TT x)
\end{equation}
where $\theta = (a_k, w_k)_{k=1}^m = (a_1, w_1, ..., a_m, w_m) \in \bR^{(d+1)m}$ is parameter and $x \in \bR^d$ is input. We call each $w_k \in \bR^d$ an input weight and each $a_k \in \bR$ an output weight. Given samples $\{(x_i, y_i) \in \bR^{d} \times \bR\}_{i=1}^n \in \bR^{nd}$, let $e_i(\theta) = e_i(g, \theta) = g(\theta, x_i) - y_i$ for each $i$ and define the $L^2$ loss function 
\[
    R(\theta) = R(g, \theta) = \sum_{i=1}^n \left( g(\theta, x_i) - y_i \right)^2. 
\]

In this paper, we will always assume that our activation $\sigma: \bR \to \bR$ is an analytic non-polynomial function. One important result for such activation is about the linear independence of neurons which we show below. This result helps us characterize the critical sets by determining the minimal width of output functions they represent. 

\begin{lemma}\label{Asmp Generic activation II}
    Let $\sigma: \bR \to \bR$ be an analytic non-polynomial. Then for any $d \in \bN$, $m \ge 2$ and any $w_1, ..., w_m \in \bR^d$, the neurons $\sigma(w_1^\TT x), ..., \sigma(w_m^\TT x)$ are linearly independent if and only if every two of them are linearly independent. 
\end{lemma}
\begin{proof}
    See the proof of \ref{ALem Generic activation II} in Appendix. 
\end{proof}

In particular, 
\begin{itemize}
    \item [(a)] If $\sigma$ does not have parity and $\sigma(0) \ne 0$, then $\sigma(w_1^\TT x), ..., \sigma(w_{m'}^\TT x)$ are linearly independent if and only if $w_{k_1} \ne w_{k_2}$ for all distinct $k_1, k_2 \in \{1, ..., m'\}$. 

    \item [(b)] If $\sigma$ does not have parity and $\sigma(0) = 0$, then $\sigma(w_1^\TT x), ..., \sigma(w_{m'}^\TT x)$ are linearly independent if and only if $w_{k_1} \ne w_{k_2}$ for all distinct $k_1, k_2 \in \{1, ..., m'\}$, or $w_k = 0$ for some $k$. 

    \item [(c)] If $\sigma$ is an odd or even function, then $\sigma(w_1^\TT x), ..., \sigma(w_{m'}^\TT x)$ are linearly independent if and only if $w_{k_1} \ne w_{k_2}$ for all distinct $k_1, k_2 \in \{1, ..., m'\}$, or when $\sigma(0) = 0$ we have an extra case $w_k = 0$ for some $k$. 
\end{itemize}

\subsection{Illustration of Main Results}\label{Subsection Illust of main results}

In this section we first introduce the main results in Section \ref{Section Criticality preserving operators} and Section \ref{Section Geometry and functional properties of critical sets} by analyzing a concrete example. Then we present the main results in an informal way. Consider a single neuron network $g_1((a', w'), x) = a' e^{w'^{\TT} x}$ with $a' \in \bR$ and $w', x \in \bR^2$. Consider four samples 
\[
    (x_1, y_1) = ((1,0),1),\, (x_2, y_2) = ((0,1), 1),\, (x_3, y_3) = ((1,1), 0),\, (x_4, y_4) = ((1,-1), 0). 
\]
and define the loss $R(g_1, \cdot)$ with respect to these samples. 
It is not difficult to see that $\theta' = (a',w') = (1,(\log\frac{1}{3}, 0))$ is the critical point of $R$ with non-zero output weight. Moreover, $R(g_1, \theta') > 0$. \\

Let $g_m(\theta, x) = \sum_{k=1}^m a_k e^{w_k^\TT x}$ for any $m \ge 3$. Up to permutations of the components of a point, the set of critical points $\theta$ for $R(g_m, \cdot)$ with $g_m(\theta, \cdot) = g_r(\theta', \cdot)$ is the union of the following sets 
\begin{align*}
    \calC^{1,m-1} &:= \left\{ (1, w', 0, w_2, ..., 0, w_m): w_k \in \calM_{\theta'}, \forall\,2 \le k \le m \right\} \\ 
    \calC^{1, m-2} &:= \left\{ (\delta, w', 1-\delta, w', 0, w_3, ..., 0, w_m): w_k \in \calM_{\theta'}, \forall\, 3\le k \le m, \delta \in \bR\cut\{0,1\} \right\} \\ 
    &\vdots \\
    \calC^{1,0} &:= \left\{ (\delta_1, w', \delta_2, w', ..., \delta_m, w'): \delta_1, ..., \delta_m \ne 0, \sum_{k=1}^m \delta_m = 1 \right\}. 
\end{align*}
Here $\calM_{\theta'}$ is the zero set of $w \mapsto \sum_{i=1}^n e_i(g_r, \theta') \sigma(w^\TT x_i)$; in particular, $\calM_{\theta'} = \bR \times \{0\}$. On the other hand, using the splitting embedding operator, we only get proper subsets of them. Moreover, we notice that i) $\calC^{1,l_1}, \calC^{1,l_2}$ are connected for all $l_1, l_2 \in \{0, 1, ..., m-1\}$ and ii) any point in $\calC^{1,l}$ with $l > 0$ is a saddle, while part of $\calC^{1,0}$ are strict saddles. See also Figure \ref{Figure Stratification C1} below.

\begin{figure}[h]
    \centering
    \includegraphics[width=0.85\textwidth]{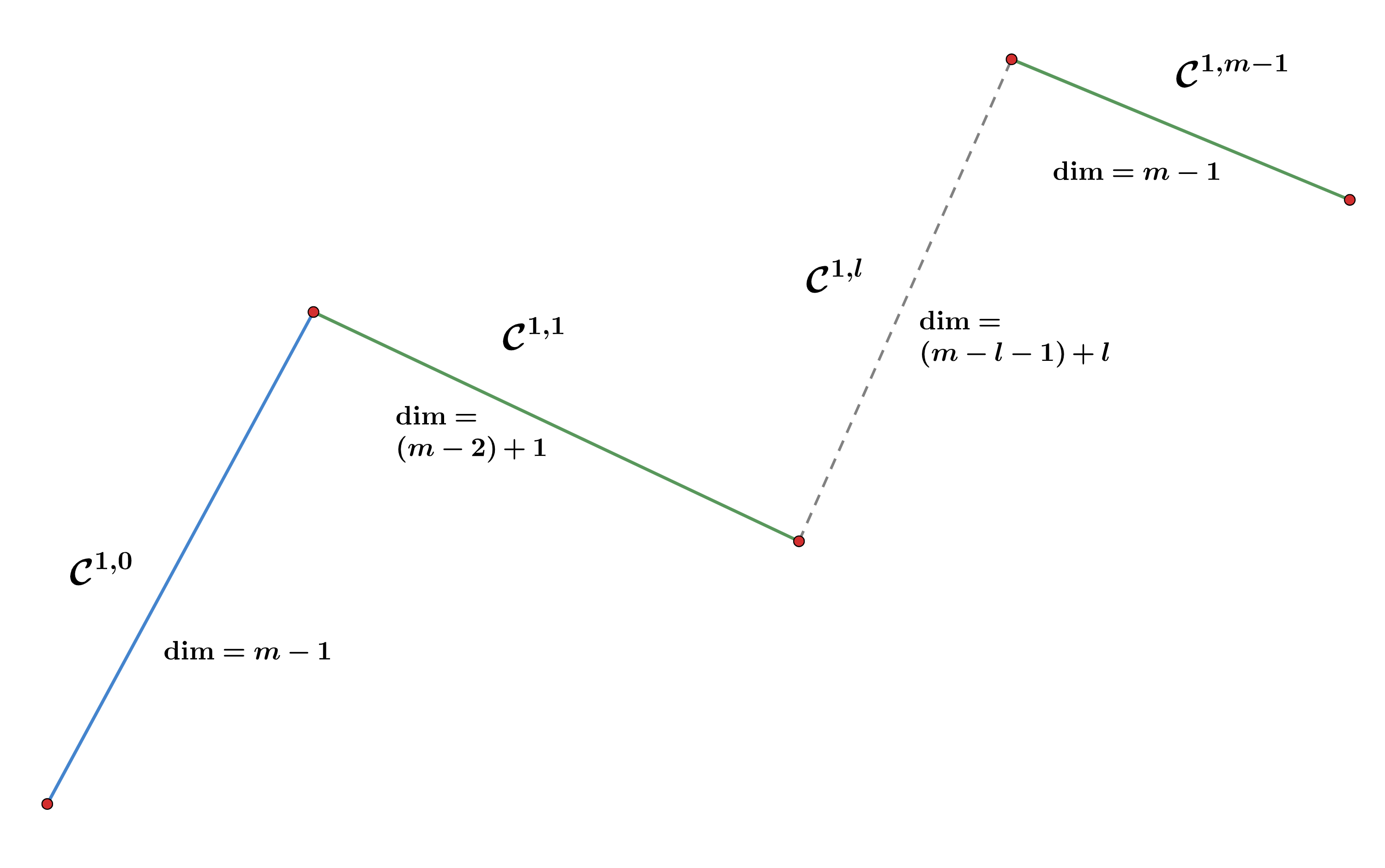}
    \caption{Illustration of $\calC^{1,l}$ for $0 \le l \le m-1$ defined above. For each $0 \le l \le m-1$, $\calC^{1,0}$ is an affine subspace of dimension $(m-l-1) + l = m-1$. The set \textcolor{blue}{$\calC^{1,0}$} has strict saddles, and points in all the other \green{$\calC^{1,l}$ with $1 \le l \le m-1$} are saddles. Moreover, these critical sets are connected to one another, as shown in this figure. See Lemma \ref{ALem for example} and its remark for a proof.}
    \label{Figure Stratification C1}
\end{figure}

In general, we have the following results for the observations above. First, we completely characterize the geometry of critical sets representing a given critical function. 

\begin{thm}[branch geometry -- informal]\label{Thm Branch geometry informal}
    The critical points of a width-$m$ neural network representing a critical network of minimal width $r \le m$, say $g_r(\theta', \cdot)$ for some $\theta' \in \bR^{(d+1)r}$, forms a finite union of sets which we call branches. Each branch is isometric to $\bR^{N} \times \Pi_{j=1}^l \calM_{\theta'}$, where $N,l \in \bN$ are independent of $\theta'$ and samples, while the manifold $\calM_{\theta'} \siq \bR^d$ depends on $\theta'$ and samples, with its dimension $\dim \calM_{\theta'} \le d-1$. Moreover, these branches are connected to each other. 
\end{thm}

For illustration, consider the example above discussing the set of critical points of a four-neuron network representing a single-neuron network. The branches are precisely $\calC^{1,m-1}, ..., \calC^{1,0}$ as well as those obtained by permuting the indices of the points in these critical sets. Moreover, in Theorem \ref{Thm Branch geometry informal} the $\bR^{N}$ is $\bR^0$, and $\calM_{\theta'}$ is just $\bR \times \{0\}$, i.e., the zero set of 
\[
    \bR^2 \ni w \mapsto \sum_{i=1}^4 (a' e^{w'^{\TT} x_i} - y_i) e^{w^\TT x_i}. 
\]

To systematically study these critical sets we present critical embedding and critical reduction operators in Section \ref{Section Criticality preserving operators}. These maps are output- and criticality-preserving. Intuitively, given a critical point $\theta \in \bR^{(d+1)m}$ with representing a network of width $r$, we can map it by critical reduction operator to a $\theta' \in \bR^{(d+1)r}$ which represent the same network. Then, using embedding operator and solving the equation (in $w$)
\[
    \sum_{i=1}^n e_i(g_r, \theta') \sigma(w^\TT x_i) = 0, 
\]
we are able to determine the critical points in $\bR^{(d+1)m}$ representing the same network as $\theta$ does. A detailed treatment can be found in Theorem \ref{Thm Embedding geometry}. Then we show the branch connectivity in Theorem \ref{Thm Connectivity of branches}. \\

Second, we demonstrate a simple relationship between critical network and saddle existence. 

\begin{thm}[existence of saddles -- informal]
    Following the notations in Theorem \ref{Thm Branch geometry informal}, the set of critical points of $R(g_m, \cdot)$ representing $g_r(\theta', \cdot)$ contains saddles\footnote{a saddle $\theta$ for a differentiable function $f$ is a point with $\nabla f(\theta) = 0$, and $\theta$ is neither a local minimum nor local maximum of it.} whenever $r < m$. 
\end{thm}

This theorem is summary of Proposition \ref{Prop Embedding saddles} and Proposition \ref{Prop Saddle branches}. In the latter proposition we show that any critical point $\theta^* = (a_k^*, w_k^*)_{k=1}^m$ with $a_k^* = 0$ for some $k$ must be a saddle. Intuitively, this can be shown in the following way: given any such critical point $\theta^*$ and some $k \in \{1, ..., m\}$ with $a_k = 0$, we perturb $w_k$ arbitrarily small to obtain some non-critical point $\ttheta$ with $R(g_m, \theta^*) = R(g_m, \ttheta)$. Since $\nabla R(g_m, \ttheta) \ne 0$, we can find some $\theta$ arbitrarily close to $\ttheta$ with $R(g_m, \theta) < R(g_m, \ttheta) = R(g_m, \theta^*)$. See also Figure \ref{Figure Find a saddle} for illustration. Similarly, there is a $\theta$ arbitrarily close to $\ttheta$ with $R(g_m, \theta) > R(g_m, \ttheta) = R(g_m, \theta^*)$. This shows $\theta^*$ is a saddle. 

\begin{figure}[h]
    \centering
    \includegraphics[width=0.65\textwidth]{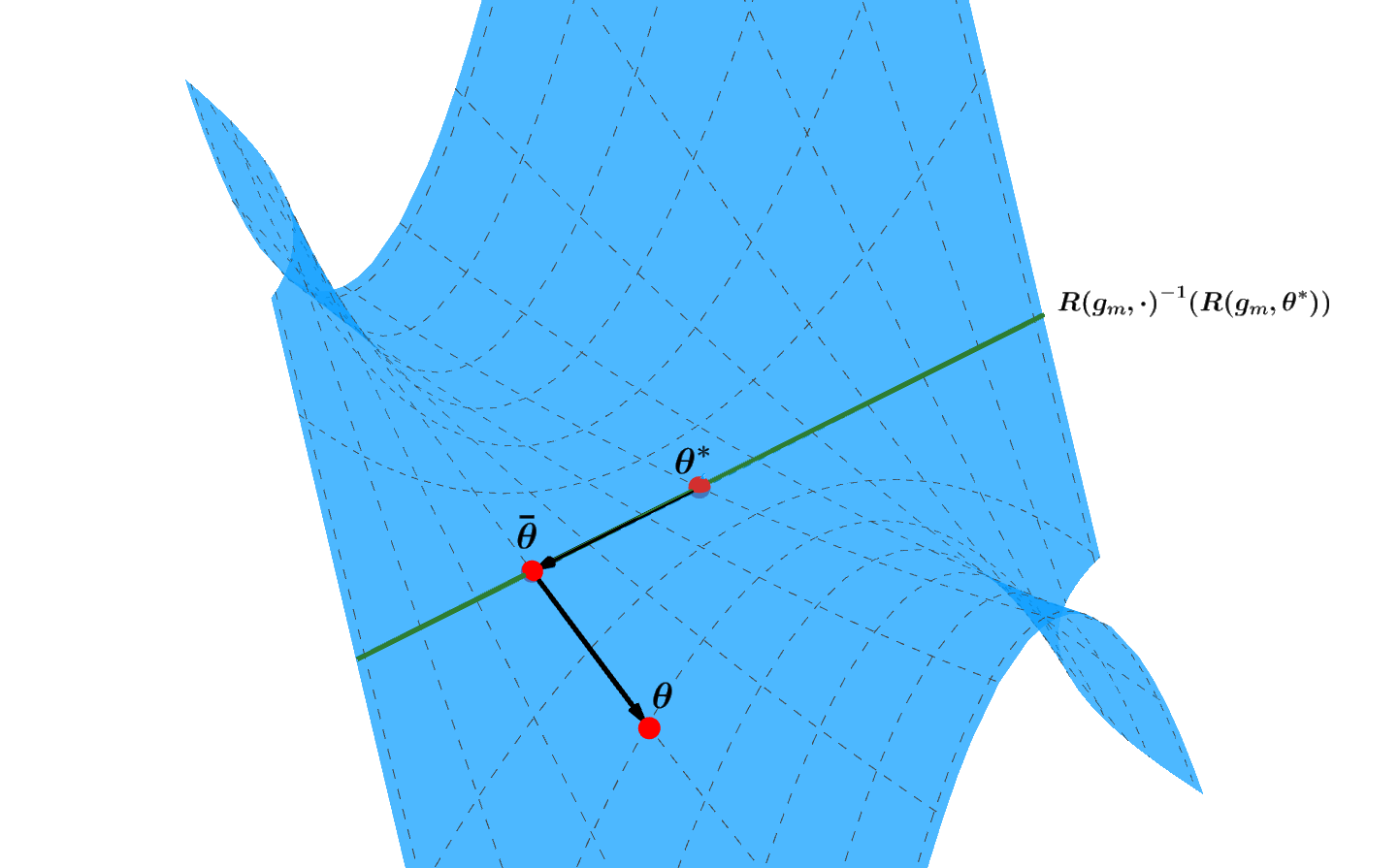}
    \caption{Illustration of our method above to show every point in $\calC^{1,l}$ with $l > 1$ is a saddle. Starting from $\theta^* \in \calC^{1,l}$, we first perturb it to $\ttheta$ with the same loss value as that of $\theta^*$, then, using the fact that $\nabla R(g_m, \theta^*) \ne 0$, we perturb it to a $\theta$ arbitrarily close to $\ttheta$ with $R(g_m, \theta) < R(g_m, \ttheta) = R(g_m, \theta^*)$.}
    \label{Figure Find a saddle}
\end{figure}

\section{Criticality Preserving Operators}\label{Section Criticality preserving operators}

We introduce two criticality preserving operators to help us study the critical sets of loss function.  The first one is the \textit{critical embedding operator}, which is introduced in \cite{EbddPrincipleShort, EbddPrincipleLong}. For completeness and our specific consideration of two-layer neural networks, we rewrite it in Definition \ref{Defn Critical embedding}. The second one, which can be viewed as a ``converse'' of critical embedding, is \textit{critical reduction operator} introduced in Definition \ref{Defn Critical reduction}. Then we show in Proposition \ref{Prop Properties of CE and CR operators} that these operators preserve output function, error terms of loss, and most importantly criticality of loss. \\

First, let's introduce an action on $\bR^{(d+1)m}$ which is induced by the symmetry properties of a neural network model. Given $m \in \bN$, recall that $S_m$ is the set of permutations on $\{1, ..., m\}$.

\begin{defn}[permutation action]
     For each $m \in \bN$ define an action $S_m \times \bR^{(d+1)m} \to \bR^{(d+1)m}$ by 
     \begin{equation}
         \pi \cdot \theta = \left( a_{\pi(k)}, w_{\pi(k)} \right)_{k=1}^m, \quad \forall \, \theta = (a_k, w_k)_{k=1}^m \in \bR^{(d+1)m}. 
     \end{equation}
\end{defn}

Intuitively, each $\pi$ acts on $\bR^{(d+1)m}$ by permuting the indices of $a_k$ and $w_k$'s, so it can be viewed as an orthogonal transformation on $\bR^{(d+1)m}$. Thus, for any $E \siq \bR^{(d+1)m}$, $\pi \cdot E = \{\pi \cdot \theta: \theta \in E\}$ is isometric to $E$. \\

We then define a stratification of parameter space by the number of features. Given $N \in \bN$, a partition of $\{1, ..., N\}$ is a sequence of numbers $P = (t_0, ..., t_r)$ such that $0 = t_0 < t_1 < ... < t_r = N$. \\

\begin{defn}[stratification of parameter space]\label{Defn Stratification of param}
    Given $m, r, l \in \bN$ and $P = (t_0, ..., t_r)$ such that $t_r + l = m$. Define $Q_P^{r,l}$ as the set of points $\theta = (a_k, w_k)_{k=1}^m \in \bR^{(d+1)m}$ such that 
    \begin{itemize}
        \item [(a)] $a_1, ..., a_{t_r} \ne 0$ and $a_{t_r+1} = ... = a_{m} = 0$. 
        \item [(b)] $\sigma\left( w_{t_{j-1}+1}^\TT(\cdot)\right), ..., \sigma\left( w_{t_j}^\TT(\cdot)\right)$ are linearly dependent for all $1 \le j \le r$. 

        \item [(c)]  $\sigma\left( w_{t_j}^\TT(\cdot)\right)$ and $\sigma\left( w_{t_{j'}}^\TT(\cdot)\right)$ are linearly independent for all distinct $j, j' \in \{1, ..., r\}$. 
    \end{itemize}
    Furthermore, define $Q_{P,\pi}^{r,l} := \pi\cdot Q_P^{r,l} = \left\{ \pi \cdot \theta: \theta \in Q_P^{r,l}\right\}$. Then define $Q^{r,l} = \cup_{P,\pi} Q_{P,\pi}^{r,l}$. 
    
    Given $r, l, P, \pi$ as above, we call $r$ the effective feature number, $l$ the ineffective feature number, and $Q_P^{r,l}$, $Q_{P,\pi}^{r,l}$ and $Q^{r,l}$ as branches of $\bR^{(d+1)m}$. When we emphasize the underlying neural network width $m$, we write $Q_P^{r,l}(m)$, $Q_{P,\pi}^{r,l}(m)$ and $Q^{r,l}(m)$, respectively. 
\end{defn}

The requirement (b) and (c) can be made explicitly for input weights. If $\sigma$ has no parity (b) says $w_k = w_{t_j}$ for all $t_{j-1} < k \le t_j$ and all $1 \le j \le r$ and (c) says $w_{t_j} \ne w_{t_{j'}}$ for all distinct $j,j' \in \{1, ..., r\}$. The case for an odd/even $\sigma$ is more complicated: one case for (b) writes $w_k = \pm w_{t_j}$ and for (c) $w_{t_j} \pm w_{t_{j'}} \ne 0$ for distinct $j,j'$; moreover, if $\sigma(0) = 0$ then $w_1, ..., w_{t_r} \ne 0$ when $r > 1$ (because $\sigma(0) = 0$ is linearly dependent with all the neurons $\sigma(w_1^\TT(\cdot)), ..., \sigma(w_{t_r}^\TT (\cdot))$), but when $r = 1$ we may have $w_k = 0$ for some $1 \le k \le t_r$. \\

Then we note that given $\theta = (a_k, w_k)_{k=1}^m \in \bN$, we can find a unique branch $Q_{P,\pi}^{r,l}$ containing $\theta$ in the following way: first count the number of $a_k$'s which are zero; this determines $l$. Second, for each $k$ with $a_k \ne 0$, find all $w_{k'}$'s such that $\sigma(w_k^\TT(\cdot))$ and $\sigma(w_{k'}^\TT (\cdot))$ are linearly dependent. This uniquely divides the set $\{k \in \{1, ..., m\}: a_k \ne 0 \}$ into several groups, thus determining $r$ which is just the number of groups. Then find a permutation $\pi \in S_m$ and a partition $P$ such that $\pi^{-1} \cdot \theta \in Q_P^{r,l}$. It follows that $\theta \in Q_{P,\pi}^{r,l}$. \\

\begin{defn}[critical embedding operator, summary from \cite{EbddPrincipleShort}]\label{Defn Critical embedding}
    Given partition $P = (t_0, ..., t_r)$ as above and an index mapping $\calI: \{1, ..., l\} \to \{1, ..., r\}$. Let $\Delta(P) = (\delta_1, ..., \delta_{t_r})$ be any vector such that $\sum_{k=t_{j-1}+1}^{t_j} \delta_k = 1$ for all $1 \le j \le r$. Then define $\iota_{P,\calI}: \bR^{(d+1)r} \to \bR^{(d+1)m}$ by 
    \begin{equation}
    \begin{aligned}
        \iota_{\Delta(P), \calI}(\theta) = &\left(\delta_1 a_1', w_1', ..., \delta_{t_1} a_1', w_1', ..., \right.\\
        &\left. \delta_{t_{r-1}+1}a_r', w_r', ..., \delta_{t_r}a_r', w_r' \right.\\
        &\left.0, w_{\calI(1)}', ..., 0, w_{\calI(l)}' \right),  
    \end{aligned}
    \end{equation}
    where $\theta = (a_k', w_k')_{k=1}^r$.  Furthermore, define $\iota_{\Delta(P), \calI, \pi}(\theta) = \pi\cdot \iota_{\Delta(P), \calI}$ for any given permutation $\pi \in S_m$. We call both $\iota_{\Delta(P), \calI}$ and $\iota_{\Delta(P), \calI, \pi}$ (two-layer) critical embedding operators. 
\end{defn}

\begin{defn}[critical reduction operator]\label{Defn Critical reduction}
    Given any branch $Q_P^{r,l}$ with $P = (t_0, ..., t_r)$, define $\varphi_P: Q_P^{r,l} \to \bR^{(d+1)r}$ by 
    \begin{equation}
        \varphi_P(\theta) = \left( \sum_{k=t_0+1}^{t_1} a_k, w_{t_1}, ..., \sum_{k=t_{r-1}+1}^{t_r} a_k, w_{t_r} \right), 
    \end{equation}
    where $\theta = (a_k, w_k)_{k=1}^m$. Similarly, define $\varphi_{P,\pi}: Q_{P,\pi}^{r,l} \to \bR^{(d+1)r}$ by $\varphi_{P,\pi}(\theta) = \varphi_P(\pi^{-1} \cdot \theta)$, where the permutation $\pi^{-1}$ is just the inverse of permutation $\pi$. 
\end{defn}

Since $\cup_{P,\pi} \cup_{r,l} Q_{P,\pi}^{r,l} = \bR^{(d+1)m}$, every $\theta \in \bR^{(d+1)m}$ is in the domain in one of such mappings. Moreover, the composition of two critical reduction operators finds the minimal network representing points in $Q_P^{r,l}$, as long as the composition is well-defined. \\

We then present the properties of the two operators. 

\begin{prop}[properties of critical embedding and critical reduction operators]\label{Prop Properties of CE and CR operators}
    Given $r, m \in \bN$ with $r \le m$, denote $g_m$ and $g_r$ as the neural networks of width $m$ and $r$, respectively (see Definition of $g$ in Section \ref{Subsection Preliminaries}). Let $P =(t_0, ..., t_r)$. For any $\theta' \in \bR^{(d+1)r}$ and $\theta \in \bR^{(d+1)m}$, The following results hold for $\iota_{\Delta(P), \calI}(\theta')$ and $\varphi_P$: 
    \begin{itemize}
        \item [(a)] (output function is preserved) $g_r(\theta', x) = g_m(\iota_{\Delta(P), \calI}(\theta'), x)$ and $g_m(\theta, x) = g_r(\varphi_P(\theta), x)$. 
        
        \item [(b)] (error-term is preserved) $e_i(g_r, \theta') = e_i(g_m, \iota_{\Delta(P), \calI}(\theta'))$ and $e_i(g_m, \theta) = e_i(g_r, \varphi_P(\theta))$. 
        
        \item [(c)] (criticality is preserved) $\nabla R(g_r, \theta') = 0$ implies $\nabla R(g_m, \iota_{\Delta(P), \calI}(\theta')) = 0$ and $\nabla R(g_m, \theta) = 0$ implies $\nabla R(g_r, \varphi_P(\theta)) = 0$. 
    \end{itemize}
\end{prop}
\begin{proof}
    The proof of results for critical embedding operator can be found in e.g. \cite{EbddPrincipleShort, EbddPrincipleLong}. For completeness we prove the results for both operators in Proposition \ref{AProp Properties of CE and CR operators}. 
\end{proof}

\section{Geometry and Functional Properties of Critical Sets}\label{Section Geometry and functional properties of critical sets}

We are now ready to study the structures of critical sets. In Section \ref{Subsection Geometry of critcal sets}, we use critical embedding and critical reduction operators to characterize the branches $\calC^{r,l}$ of $\critR$ (see Definition \ref{Defn Branches of critR} below for $\calC^{r,l}$), including the covering property, structure of the critical set representing a given output function, the connectivity of these branches, as well as their dimensions. These geometrical properties have a hierarchical dependence on the effective feature number $r$ and ineffective feature number $l$. Based on the characterization of branches, in Section \ref{Subsection Saddle and saddle connectivity} we investigate the relationship between saddle existence and minimal width of output function. We introduce two types of saddle. One is called ``embedding saddles'' obtained directly from applying the critical embedding operator (Proposition \ref{Prop Embedding saddles}). The other one occupies every branch $\calC^{r,l}$ with $l > 0$, categorizing it as a ``saddle branch'' (Proposition \ref{Prop Saddle branches}). \\

First, let's group critical points with non-zero loss value into different ``branches" according to the stratification of parameter space (Definition \ref{Defn Stratification of param}. 

\begin{defn}[branches of critical set]\label{Defn Branches of critR}
    Given $l \le m$, $r \le m - l$ and a partition $P = (t_0, t_1, ..., t_r)$ with $t_0=0, t_r=m-l$, we denote $\calC_P^{r,l}$ (``$\calC$" stands for ``critical") as the subset consisting of critical points $\theta = (a_k, w_k)_{k=1}^m$ in $Q_P^{r,l}$. Furthermore, define $\calC_{P,\pi}^{r,l} := \pi \cdot \calC_P^{r,l}$ for any given permutation $\pi \in S_m$. We call $\calC_P^{r,l}$, $\calC_{P,\pi}^{r,l}$ and $\calC^{r,l}$ (critical) branches of $\critR$. When we emphasize the underlying neural network width $m$, we write $\calC_P^{r,l}(m)$, $\calC_{P,\pi}^{r,l}(m)$ and $\calC^{r,l}$, respectively. 
\end{defn}
\begin{remark}
    By definition, for each permutation $\pi \in S_m$, $\calC_P^{r,l}$ and $\calC_{P,\pi}^{r,l}$ are isometric, indicating that $\calC^{r,l}$ is a finite union of isometric subsets for any given $r$ and $l$. Furthermore, this means to study the $\calC_{P,\pi}^{r,l}$'s we only need to study the $\calC_P^{r,l}$'s. 
\end{remark}

\subsection{Geometry of Critical Sets}\label{Subsection Geometry of critcal sets}

\begin{thm}[embedding geometry]\label{Thm Embedding geometry}
    Given effective feature numbers $r$ and ineffective feature numbers $l$ as in Definition \ref{Defn Branches of critR}, we have the following results. 
     \begin{itemize}
        \item [(a)] (covering) The set of non-global critical points of $R$, $\critR \cut \Rzero$, is the disjoint union of $\calC^{r,l}$'s, i.e., $\critR \cut \Rzero = \cup_{r,l} \calC^{r,l}$. 

        \item [(b)] (embedding structure) For any partition $P = (t_0, ..., t_r)$ such that $t_r + l = m$ and any index mapping $\calI: \{1, ..., l\} \to \{1, ..., r\}$, we have 
        \[
            \iota_{\Delta(P), \calI}(\calC^{r,0}(r)) \siq \calC_P^{r,l}(m). 
        \]
        Conversely, $\varphi_P$ induces a map 
        \[
            \varphi: \calC_P^{r,l}(m) \to \coprod_{r'=0}^r \calC^{r', r-r'}(r) \siq \bR^{(d+1)r}
        \]
        such that the inverse image of every $\theta'$, $\overline{\varphi^{-1}(\theta')}$ is a finite union of sets, each one isometric to $\bR^{t_r-r} \times \coprod_{j=1}^l \calM_{\theta'}$. Here $\calM_{\theta'} \siq \bR^d$ is an analytic set determined by $\theta'$ of dimension at most $d-1$. 
    \end{itemize}
\end{thm}
\begin{proof}
    See the proof of \ref{AThm Embedding Geometry} in Appendix. 
\end{proof}
\begin{remark}
    When $\sigma$ is also an odd function, $\calM_{\theta'}$ is always an analytic set of dimension $d-1$. See Lemma \ref{ALem Dim CalM}. 
\end{remark}

By Theorem \ref{Thm Embedding geometry} (a), we can see that to study the critical points of $R$ (with non-zero loss), it suffices to study all the branches $\calC^{r,l}$'s. Then (b) tells us each such branch has a simple, hierarchical structure depending on the effective feature number $r$ and ineffective feature number $l$. In fact, more can be deduced from (b). First, for every critical point $\theta$, the set $\calC_{\theta}$ of critical points representing the same output function as $\theta$ is a finite union of sets taking the form $\overline{\varphi_{P',\pi'}^{-1}(\varphi_{P,\pi}(\theta))}$'s. More precisely, if $\theta \in \calC_{P,\pi}^{r,l}$ and $\theta' := \varphi_{P,\pi}(\theta)$,  
\begin{equation}
    \bigcup_{l=0}^{m-r} \bigcup_{P',\pi'}  \overline{\varphi_{P',\pi'}^{-1}(\theta')} = \calC_{\theta} = \{\ttheta \in \critR: g(\ttheta, \cdot) = g(\theta, \cdot)\},  
\end{equation}
so $\calC_{\theta'}$ is a finite union of sets, each one isometric to a Euclidean space product several identical analytic set\footnote{An analytic set is the common zero set of a finite collection of (real) analytic functions.} of dimension at most $d-1$. \\

Second, (b) helps us analyze the dimension of each branch $\calC^{r,l}$. Given effective feature number $r \le m$, let $N := \dim (\nabla R(g_r, \cdot))^{-1}(0)$ be the dimension of the critical points for width-$r$ network $g_r$. From Theorem \ref{Thm Embedding geometry} (b), we can see that for each ineffective feature number $l$, $\dim \calC^{r,l} \le (m-l-r) + N + l(d-1)$ and in particular when $\sigma$ is an odd activation satisfying assumption \ref{Asmp Generic activation II}, $\dim \calC^{r,l} = (m-l-r) + N + l(d-1)$. Thus, when $d=1$, $\dim \calC^{r,l}$ decreases as $l$ increases; when $d=2$ and $\sigma$ is odd, $\dim \calC^{r,l_1} = \dim \calC^{r,l_2}$ as long as both sets are non-empty; when $d > 2$ and $\sigma$ is odd, $\dim \calC^{r,l}$ increases as $l$ increases. See also figure \ref{Figure Stratification Cr}

\begin{thm}[connectivity of branches]\label{Thm Connectivity of branches}
   The following results hold for branches of $\critR$ defined as in Definition \ref{Defn Branches of critR}. 
    \begin{itemize}
        \item [(a)] Given effective feature number $r_1$ and a partition $P_1 = (t_0, ..., t_{r_1})$ with $t_{r_1} + l_1 = m$. Let $s_1$ count the number of elements in the set $\{j: t_j - t_{j-1} > 0\}$. If the ineffective feature numbers $l_1, l_2$ satisfy $l_1 < l_2$, $r_1 + l_1 < m$ and $r_2 + l_2 \le m$, then there are partitions $P_{2,1}, ..., P_{2,s_1}$ and permutations $\pi_{2,1}, ..., \pi_{2, s_1}$ such that $\overline{\calC_{P_1}^{r_1,l_1}} \cap \calC_{P_{2,j}, \pi_{2,j}}^{r_1, l_2} \ne \emptyset$ for all $1 \le j \le s_1$. 

        \item [(b)] Suppose the effective feature numbers $r_1, r_2$ with $r_1 < r_2$. Given a partition $P_1$ as in (a), then there are ineffective feature numbers $l_1, l_2$, a partition $P_2$ and a permutation $\pi_2$ giving $\calC_{P_1}^{r_1,l_1} \cap \overline{\calC_{P_2, \pi_2}^{r_2,l_2}} \ne \emptyset$, if the following requirements are satisfied. 
        \begin{itemize}
            \item [i)]    $l_1 > 0$. 
            \item [ii)]   $2r_2 - r_1 + l \le m$. 
            \item [iii)]  For some $\theta \in \calC_{P_1}^{r_1,l_1}$ the common zero set of the $(d+1)$-many maps 
            \begin{align*}
                &w \mapsto \sum_{i=1}^n e_i(\theta) \sigma(w^\TT x_i) \\ 
                &w \mapsto \sum_{i=1}^n e_i(\theta) \sigma'(w^\TT x_i) (x_i)_1 \\ 
                &\,\, \vdots \\
                &w \mapsto \sum_{i=1}^n e_i(\theta) \sigma'(w^\TT x_i) (x_i)_d 
            \end{align*}
            has at least $r_2$-many elements. 
        \end{itemize}
    \end{itemize}
\end{thm}
\begin{proof}
    See the proof of Theorem \ref{AThm Connectivity of branches} in Appendix. 
\end{proof}

\begin{figure}[h]
    \centering
    \includegraphics[width = 0.85\textwidth]{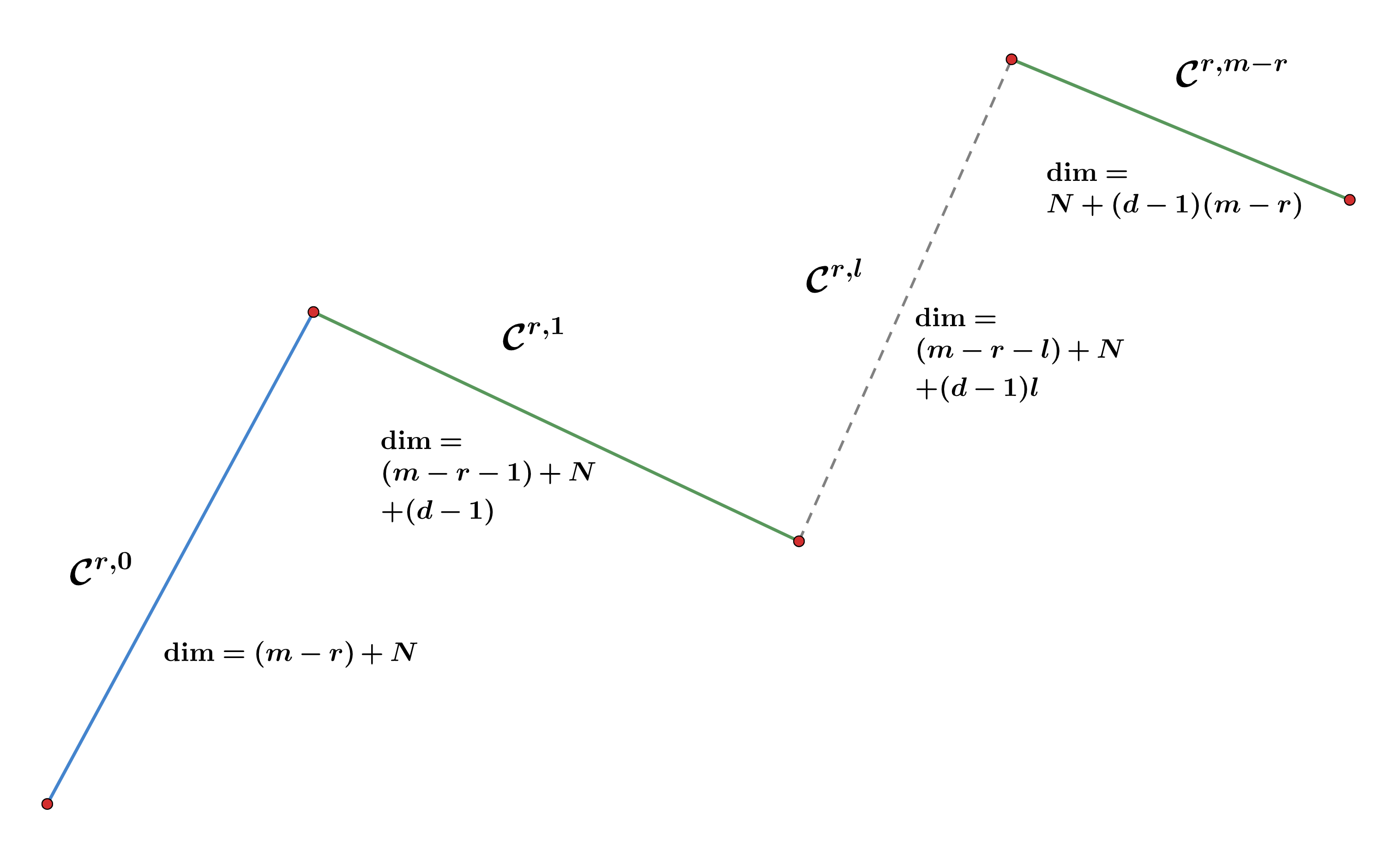}
    \caption{As illustrated by the figure, $\calC^{r,0}, \calC^{r,1}, ..., \calC^{r,m-r}$ are connected to one another, and when $\sigma$ for each $1 \le l \le m$ the branch $\calC^{r,l}$ is an analytic set of dimension $(m-l-r) + N + (d-1)$ with $N = \dim (\nabla R(g_r, \cdot))^{-1}(0)$. Furthermore, each \green{$\calC^{r,l}$ with $1 \le l \le m-r$ consist only of saddles}, while the branch \textcolor{blue}{$\calC^{r,0}$} has strict saddles provided that the hypothesis in \ref{Prop Embedding saddles} is satisfied. } 
    \label{Figure Stratification Cr}. 
\end{figure}

\subsection{Saddle and Saddle Connectivity}\label{Subsection Saddle and saddle connectivity}

\begin{prop}[Embedding saddles]\label{Prop Embedding saddles}
    Let $r < m$. Suppose that $\theta' = (a_k', w_k')_{k=1}^r \in \calC_{P^*}^{r,0}(r)$ is a critical point of $R$ and the matrices
    \[
        \sum_{i=1}^n e_i(\theta') \sigma''(w_j'^\TT x_i) 
        \maThree{(x_i)_1(x_i)_1}{...}{(x_i)_1(x_i)_d}{\vdots}{\ddots}{\vdots}{(x_i)_d(x_i)_1}{...}{(x_i)_d(x_i)_d} 
    \]
    and $\sum_{i=1}^n \sigma(w_l'^\TT x_i)^2$ are non-zero for some $j$ and $l$. Then there is a partition $P = (t_0, ..., t_r = m)$ such that $\calC_P^{r,0}(m)$ contains strict saddle points $\theta = (a_k, w_k)_{k=1}^m$ satisfying $g(\theta, \cdot) = g(\theta', \cdot)$. 
\end{prop}
\begin{proof}
    See the proof of Proposition \ref{AProp Embedding saddles}
\end{proof}

\begin{prop}[Saddle branches]\label{Prop Saddle branches}
    Let $r < m$. Any branch $\calC^{r,l}$ with $l > 0$ consists only of saddle points. Moreover, for any $\theta_0 \in \calC^{r,0}$, $\theta_0$ is connected to a saddle $\theta_1$ via a line segment $\gamma: [0,1] \to \bR^{(d+1)m}$, such that $\gamma(t) \in \critR$ and $g(\gamma(t), \cdot) = g(\theta_0, \cdot)$ for all $t$. 
\end{prop}
\begin{proof}
    See the proof of Proposition \ref{AProp Saddle branches}. See also Figure \ref{Figure Stratification Cr} for illustration. 
\end{proof}

\section{Conclusion}

In this paper we characterize the geometry of critical sets and the existence of saddle branches of a neural network. The analysis is based on a stratification of parameter space according to the width of output functions they represent. First we present two criticality preserving maps, namely the critical embedding operator and critical reduction operator. Then we categorize the critical points with non-zero loss value into branches according the stratification of parameter space, and characterize the geometry of these branches with the two operators. We uncover some of their simple geometry, meanwhile showing the covering property, connectivity and dimensions of them. Finally we prove that whenever the minimal width of such output function is smaller than that of our model, the critical set representing it always has saddle branches, and often strict saddles as well. In general, the paper provides a foundation for more detailed investigation of neural network optimization and training behaviors. Future works may explore which branches are generically non-empty/empty as well as how the structures of these branches impact gradient dynamics. \\

\section{Acknowledgement}

This work is sponsored by the National Key R\&D Program of China Grant No. 2022YFA1008200 (T. L., Y. Z.), the National Natural Science Foundation of China Grant No. 12101401 (T. L.), No. 12101402 (Y. Z.),  Shanghai Municipal Science and Technology Key Project No. 22JC1401500 (T. L.), the Lingang Laboratory Grant No.LG-QS-202202-08 (Y. Z.), Shanghai Municipal of Science and Technology Major Project No. 2021SHZDZX0102.

\bibliography{main}

%%%%%%%%%%%%%%%%%%%%%%%%%%%%%%%%%%%%%%%%%%%%%%%%%%%%%%%%%%%%%5%%%
\clearpage
\appendix

\section{Appendix}

\begin{lemma}[Lemma \ref{Asmp Generic activation II}]\label{ALem Generic activation II}
    Let $\sigma: \bR \to \bR$ be an analytic non-polynomial. Then for any $d \in \bN$, $m \ge 2$ and any $w_1, ..., w_m \in \bR^d$, the neurons $\sigma(w_1^\TT x), ..., \sigma(w_m^\TT x)$ are linearly independent if and only if every two of them are linearly independent. 
\end{lemma}
\begin{proof}
    First assume that $\sigma^{(s)}(0) \ne 0$ for infinitely many even and odd integer $s$'s, and $\sigma(0) \ne 0$, then by \cite{BSimsek}, we know that for any $m,d \in \bN$ and any $w_1, ..., w_m \in \bR^d$, $\sigma(w_1^\TT x), ...., \sigma(w_m^\TT x)$ are linearly independent if and only if $w_1, ..., w_m$ are distinct; in particular, if and only if every two of them are linearly independent. A similar argument holds when $\sigma(0) = 0$, except that now $\sigma(0x) = 0$ and thus we need to take $w_k = 0$ into consideration as well. \\

    Second, assume that $\sigma$ is an even or odd function, then there are even numbers $\{s_j\}_{j=1}^\infty \siq \bN$ such that $\sigma^{(s_j)}(0) \ne 0$ for all $j \in \bN$. Given $w_1, ..., w_m \in \bR^d$, suppose that $\sigma(w_1^\TT x), ..., \sigma(w_m^\TT x)$ are linearly dependent, we must claim that either $w_k \pm w_j = 0$ for some distinct $k, j \in \{1, ..., m\}$ or when $\sigma(0) = 0$, $w_k = 0$ for some $k \in \{1, ..., m\}$. Indeed, if this is not true, there is a vector $v \in \bR^d$ such that $w_k^\TT v \pm w_j^\TT v \ne 0$ for all distinct $k, j \in \{1, ..., m\}$ and when $\sigma(0) = 0$, $w_k^\TT v \ne 0$ for all $k$. Since $\sigma(w_1^\TT x), ..., \sigma(w_m^\TT x)$ are linearly independent, there are constants $a_1, ..., a_m \in \bR$ not all zero, such that $\sum_{k=1}^m a_k \sigma(w_k^\TT x) = 0$ for all $x \in \bR^d$. Therefore, 
    \[
        \sum_{k=1}^m a_k \sigma((w_k^\TT v) z) = 0, \quad \forall\, z \in \bR
    \]
    Rewriting this in power series expansion near the origin, we obtain 
    \begin{equation}\label{eq 1 for ALem Generic activation II}
        \sum_{k=1}^m a_k \sigma((w_k^\TT v) z) = \sum_{s=0}^\infty \alpha_s \left( \sum_{k=1}^m a_k (w_k^\TT v)^s \right) z^s = 0
    \end{equation}
    for all $z$ sufficiently close to $0$, where $\alpha_s = \sigma^{(s)}(0)$ for each $s \in \bN \cup \{0\}$. Let $k_1 \in \{1, ..., m\}$ be such that $|w_{k_1}^\TT v| > |w_k^\TT v|$ for all $k$. If $a_{k_1} \ne 0$, we have 
    \[
        \sum_{k=1}^m a_k \left( w_k^\TT v \right)^s = \Theta ( w_{k_1}^\TT v)^s \to \infty 
    \]
    as $s \to \infty$; in particular, $\alpha_{s_j} \sum_{k=1}^m a_k (w_k^\TT v)^{s_j} \ne 0$ for sufficiently large $j$, which contradicts equation (\ref{eq 1 for ALem Generic activation II}). Thus, $a_{k_1} = 0$. By repeating this procedure we can see that $a_1 = ... = a_m = 0$, contradicting our assumption on the linear dependence of these neurons. Therefore, we conclude that if $\sigma(w_1^\TT x), ..., \sigma(w_m^\TT x)$ are linearly dependent, either 
    \begin{itemize}
        \item [(a)] $w_k \pm w_j = 0$ for some distinct $k, j \in \{1, ..., m\}$, which means the neurons $\sigma(w_k^\TT x)$ and $\sigma(w_j^\TT x)$ are linearly dependent; or 
        \item [(b)] $w_k = 0$ for some $k$ provided that $\sigma(0) = 0$, which means the constant-zero neuron $\sigma(w_k^\TT x)$ is linearly independent with any other neurons. \\
    \end{itemize}

    Conversely, if $\sigma(w_1^\TT x), ..., \sigma(w_m^\TT x)$ are linearly independent, then trivially every two of them are linearly independent. This completes the proof. 
\end{proof}

\begin{prop}[Proposition \ref{Prop Properties of CE and CR operators}]\label{AProp Properties of CE and CR operators}
    Given $r, m \in \bN$ with $r \le m$, denote $g_m$ and $g_r$ as the neural networks of width $m$ and $r$, respectively, i.e., 
    \begin{align*}
        g_m(\theta, x) &= \sum_{k=1}^m a_k \sigma(w_k^\TT x), \quad \forall\,\theta = (a_k, w_k)_{k=1}^m \in \bR^{(d+1)m}, \forall\,x \in \bR^d \\ 
        g_r(\theta', x) &= \sum_{k=1}^r a_k' \sigma(w_k'^{\TT} x), \quad \forall\,\theta' = (a_k', w_k')_{k=1}^r \in \bR^{(d+1)r}, \forall\, x\in \bR^d
    \end{align*}
    Let $P =(t_0, ..., t_r)$. For any $\theta' \in \bR^{(d+1)r}$ and $\theta \in \bR^{(d+1)m}$, The following results hold for $\iota_{\Delta(P), \calI}(\theta')$ and $\varphi_P$: 
    \begin{itemize}
        \item [(a)] (output function is preserved) $g_r(\theta', x) = g_m(\iota_{\Delta(P), \calI}(\theta'), x)$ and $g_m(\theta, x) = g_r(\varphi_P(\theta), x)$. 
        
        \item [(b)] (error-term is preserved) $e_i(g_r, \theta') = e_i(g_m, \iota_{\Delta(P), \calI}(\theta'))$ and $e_i(g_m, \theta) = e_i(g_r, \varphi_P(\theta))$. 
        
        \item [(c)] (criticality is preserved) $\nabla R(g_r, \theta') = 0$ implies $\nabla R(g_m, \iota_{\Delta(P), \calI}(\theta')) = 0$ and $\nabla R(g_m, \theta) = 0$ implies $\nabla R(g_r, \varphi_P(\theta)) = 0$. 
    \end{itemize}
\end{prop}
\begin{proof}
    The key to proving the results for both operators is that they do not ``introduce new input weights''. The three parts (a), (b), (c) then follows from straightforward computation. Let $\theta' =: (a_j', w_j')_{j=1}^r$ and $\theta =: (a_k, w_k)_{k=1}^m$. 
    \begin{itemize}
        \item [(a)] For $\iota_{\Delta(P), \calI}$ we use the fact that $\sum_{k=t_{j-1}+1}^{t_j} \delta_k = 0$ for all $1 \le j \le r$ to deduce that for any $x \in \bR^d$, 
        \begin{align*}
            g_r(\theta', x) 
            &= \sum_{j=1}^r a_j' \sigma(w_k'^{\TT} x) \\
            &= \sum_{j=1}^r \left( \sum_{k=t_{j-1}+1}^{t_j} \delta_k \right) a_j' \sigma(w_k'^{\TT} x) \\ 
            &= \sum_{j=1}^r \left[ \delta_{t_{j-1}+1}a_j' \sigma(w_j'^{\TT} x) + ... + \delta_{t_j}a_j' \sigma(w_j'^{\TT} x) \right] \\ 
            &= g_m(\iota_{\Delta(P), \calI}(\theta'), x). 
        \end{align*}
        Similarly, for $\varphi_P$ we use the fact that any neuron with zero output weight does not affect the output of the neural network to deduce that for any $x \in \bR^d$, 
        \begin{align*}
            g_m(\theta, x) = \sum_{k=1}^m a_k \sigma(w_k^\TT x)
            &= \sum_{k=1}^{t_r} a_k \sigma(w_k^\TT x) \\ 
            &= \sum_{j=1}^r \left( \sum_{k=t_{j-1}+1}^{t_j} a_k \right) \sigma(w_{t_j}^\TT x) \\ 
            &= g_r(\varphi_P(\theta), x). 
        \end{align*}

        \item [(b)] By (a), $g_r(\theta', x_i) - y_i = g_m(\iota_{\Delta(P), \calI}(\theta'), x_i) - y_i$ and $g_m(\theta, x_i) - y_i = g_r(\varphi_P(\theta), x_i) - y_i$ for all $1 \le i \le n$, so the desired result follows. 

        \item [(c)] For $\iota_{\Delta(P), \calI}(\theta')$: denote $(a_k'', w_k'')_{k=1}^m = \iota_{\Delta(P), \calI}(\theta')$, given $1 \le k \le m$, $w_k'' = w_j'$ and $a_j'' = \delta a_j$ for some $1 \le j \le r$ and $\delta \in \bR$, whence by (b), 
        \begin{align*}
            \parf{R(g_m, \cdot)}{a_k}(\iota_{\Delta(P), \calI}(\theta')) 
            &= 2 \sum_{i=1}^n \left( g_m(\iota_{\Delta(P), \calI}(\theta'), x_i) - y_i \right) \sigma(w_k''^{\TT}x_i) \\ 
            &= 2 \sum_{i=1}^n (g_r(\theta', x_i) - y_i) \sigma(w_j'^{\TT} x_i) = \parf{R(g_r, \cdot)}{a_j}
        \end{align*}
        and for each $1 \le t \le d$, 
        \begin{align*}
            \parf{R(g_m, \cdot)}{(w_k)_t}(\iota_{\Delta(P), \calI}(\theta')) 
            &= 2 a_k'' \sum_{i=1}^n \left( g_m(\iota_{\Delta(P), \calI}(\theta'), x_i) - y_i \right) \sigma'(w_k''^{\TT}x_i)(x_i)_t \\ 
            &= 2 \delta a_j \sum_{i=1}^n (g_r(\theta', x_i) - y_i) \sigma'(w_j'^{\TT} x_i)(x_i)_t \\ 
            &= \parf{R(g_r, \cdot)}{a_j}. 
        \end{align*}
        From these equations we can see that $\nabla R(g_r, \theta') = 0$ implies $\nabla R(g_m, \iota_{\Delta(P), \calI}(\theta') = 0$. 
        
        Similarly, for $\varphi_P(\theta)$: denote $(a_j'', w_j'')_{j=1}^r = \varphi_P(\theta)$, given $1 \le j \le r$, $w_j'' = w_{t_j}$, whence by (b), 
        \begin{align*}
            \parf{R(g_r, \cdot)}{a_j}(\varphi_P(\theta)) 
            &= 2 \sum_{i=1}^n \left( g_r(\varphi_P(\theta), x_i) - y_i\right) \sigma(w_j''^{\TT} x_i) \\ 
            &= 2 \sum_{i=1}^n (g_m(\theta, x_i) - y_i) \sigma(w_{t_j}^\TT x_i) \\ 
            &= \parf{R(g_m, \cdot)}{a_{t_j}}
        \end{align*}
        and for each $1 \le t \le d$, 
        \begin{align*}
            \parf{R(g_r, \cdot)}{(w_j)_t}(\varphi_P(\theta)) 
            &= 2 a_j'' \sum_{i=1}^n \left( g_r(\varphi_P(\theta), x_i) - y_i\right) \sigma'(w_j''^{\TT} x_i) \\ 
            &= 2 \sum_{k=t_{j-1}}^{t_j} a_k \sum_{i=1}^n (g_m(\theta, x_i) - y_i) \sigma'(w_{t_j}^\TT x_i) \\ 
            &= \sum_{k=t_{j-1}+1}^{t_j} \parf{R(g_m, \cdot)}{(w_k)_t}. 
        \end{align*}
        From these equations we can see that $\nabla R(g_m, \theta) = 0$ implies $\nabla R(g_r, \varphi_P(\theta)) = 0$.
    \end{itemize}
\end{proof}

\begin{thm}[Theorem \ref{Thm Embedding geometry}]\label{AThm Embedding Geometry}
    Given effective feature numbers $r, r_1, r_2$ and ineffective feature numbers $l, l_1, l_2$ as in Definition \ref{Defn Branches of critR}, we have the following results. 
     \begin{itemize}
        \item [(a)] (covering) The set of non-global critical points of $R$, $\critR \cut \Rzero$, is the disjoint union of $\calC^{r,l}$'s, i.e., $\critR \cut \Rzero = \cup_{r,l} \calC^{r,l}$. 

        \item [(b)] (embedding structure) For any partition $P = (t_0, ..., t_r)$ such that $t_r + l = m$ and any index mapping $\calI: \{1, ..., l\} \to \{1, ..., r\}$, we have 
        \[
            \iota_{\Delta(P), \calI}(\calC^{r,0}(r)) \siq \calC_P^{r,l}(m). 
        \]
        Conversely, $\varphi_P$ induces a map 
        \[
            \varphi: \calC_P^{r,l}(m) \to \coprod_{r'=0}^r \calC^{r', r-r'}(r) \siq \bR^{(d+1)r}
        \]
        such that the inverse image of every $\theta'$, $\overline{\varphi^{-1}(\theta')}$ is a finite union of sets, each one isometric to $\bR^{t_r-r} \times \Pi_{j=1}^l \calM_{\theta'}$. Here $\calM_{\theta'} \siq \bR^d$ is an analytic set determined by $\theta'$ of dimension at most $d-1$. 
    \end{itemize}
\end{thm}
\begin{proof}
\begin{itemize}
    \item [(a)] According to the remark after the definition of stratification of $\bR^{(d+1)m}$ (Definition \ref{Defn Stratification of param}), for any critical point $\theta \in \critR$, there is an effective feature number $r$ and ineffective feature number $l$ with $\theta \in Q^{r,l}$. Thus, if we also have $R(\theta) \ne 0$ then $\theta \in \calC^{r,l}$. This shows the covering property of these branches. 
    \item [(b)] The fact that $\iota_{\Delta(P), \calI}(\calC^{r,0}) \siq \calC_P^{r,l}$ follows from Proposition \ref{Prop Properties of CE and CR operators}. 
    
    For the converse, fix any $\theta \in \calC_P^{r,l}$ and denote 
    \[
        \theta' = (a_j', w_j')_{j=1}^r = \varphi(\theta).
    \]
    By Proposition \ref{Prop Properties of CE and CR operators}, $\theta'$ is also a critical point of $R(g', \cdot)$. Moreover, since the input weights of $\theta'$ are distinct, $\theta' \in Q^{r', r-r'}(r)$ for some $0 \le r' \le r$ (in fact, if $\theta = (a_k, w_k)_{k=1}^m$, $r'$ is the number of $j$'s such that $\sum_{k=t_{j-1}+1}^{t_j} a_k \ne 0$). This proves the first part of the result. \\
    
    We then show that there is an (affine) isometry $\calT: \overline{\varphi^{-1}(\theta')} \to \bR^{t_r-r} \times \Pi_{j=1}^l \calM_{\theta'}$. First, let $E \siq \bR^m$ be a subspace consisting of vectors $v = (\delta_k)_{k=1}^m$ such that $\sum_{k=t_{j-1}+1}^{t_j} \delta_j = 0$ for all $1 \le j \le r$ and $\delta_k = 0$ for all $k > t_r$. Second, let $\calM_{\theta'}$ be the zero set of the analytic function 
    \begin{equation}
        w \mapsto \sum_{i=1}^n e_i(\theta') \sigma(w^\TT x_i), 
    \end{equation}
    where each $e_i(\theta') = \sum_{j=1}^r a_j' \sigma(w_j'^{\TT} x_i) - y_i$. Clearly, for any other $\ttheta = (\ta_k, \tw_k)_{k=1}^m \in \calC_P^{r,l}(m)$, $\varphi(\ttheta) = \theta'$ if and only if 
    \begin{itemize}
        \item i)    $\sum_{k=t_{j-1}+1}^{t_j} \ta_k = a_j = \sum_{k=t_{j-1}+1}^{t_j} a_k$, if and only if $\sum_{k=t_{j-1}+1}^{t_j} (\ta_k - a_k) = 0$. 
        \item ii)   $w_k = w_j'$ for all $t_{j-1} < k \le t_j$ and all $1 \le j \le r$. 
        \item iii)  For each $k \in \{t_r+1, ..., m\}$, $\parf{R}{a_k}(\ttheta) = 0$, because $a_k = 0$ automatically yields $\parf{R}{w_k}(\ttheta) = 0$. But we have 
        \[
            \parf{R}{a_k}(\ttheta) = \sum_{i=1}^n e_i(\theta') \sigma(w_k^\TT x_i), 
        \]
        whence $w_k \in \calM_{\theta'}$. 
    \end{itemize}
    Therefore, up to a rearrangement $\calT_1$ of the components of points in $\bR^{(d+1)m}$ we have $\overline{\varphi^{-1}(\theta')} = (E+A_{\theta'}) \times \Pi_{k=1}^l \calM_{\theta'}$, where 
    \[
        A_{\theta'} = \left(a_1', \underbrace{0, ..., 0}_{\text{$(t_1-1$-many 0's}}, ...., a_r', \underbrace{0, ..., 0}_{\text{$(t_r-1)$-many 0's}}\right). 
    \]
    Note that $\dim E = t_r-r$, so there is an linear isometry $\calT_2: E+A_{\theta'} \to \bR^{t_r-r}$. Finally, let $\calT = (\calT_2, id) \circ \calT_1$ ($id$ is the identity map on $\Pi_{j=1}^l \calM_{\theta'}$. Then clearly $\calT$ is an isometry between $\overline{\varphi^{-1}(\theta')}$ and $\bR^{t_r-r} \times \Pi_{j=1}^l \calM_{\theta'}$. 
\end{itemize}
\end{proof}

\begin{lemma}\label{ALem Dim CalM}
    Let $\sigma$ be an odd, analytic, non-polynomial activation. Given $n$ inputs $x_1, ..., x_n$ such that $x_{i_1} \pm x_{i_2} \ne 0$ for all distinct $i_1, i_2 \in \{1, ..., n\}$ and $x_i \ne 0$ for all $i$. For any $e_1, ..., e_n$ not all zero, the zero set of the function 
    \[
        \tau: \bR^d \ni w \mapsto \sum_{i=1}^n e_i \sigma(w^\TT x_i) 
    \]
    is an analytic set of dimension $d-1$. 
\end{lemma}
\begin{proof}
    It is clear that $\sigma(w^\TT x_{i_1})$ and $\sigma(w^\TT x_{i_2})$ are linearly independent if and only if $i_1 = i_2$. By Lemma \ref{Asmp Generic activation II}, this means the neurons (in $w \in \bR^d$) $\sigma(w^\TT x_1), ..., \sigma(w^\TT x_n)$ are linearly independent. Thus, there is some $w \in \bR^d$ with $\sum_{i=1}^n e_i \sigma (w^\TT x_i) \ne 0$, say it is greater than zero. Because $\sigma$ is odd, we have 
    \[
        \sum_{i=1}^n e_i \sigma(-w^\TT x_i) = - \sum_{i=1}^n e_i \sigma(w^\TT x_i) < 0. 
    \]

    Since $\sigma(0) = 0$, it is clear that $\tau^{-1}(0) \ne \emptyset$. Assume that $\tau^{-1}(0)$ is an analytic set of dimension no greater than $d-2$. Then in particular it is a countable union of analytic submanifolds in $\bR^d$ of dimension $\le d-2$. By \cite{Hatcher}, the space $\bR^d \cut \tau^{-1}(0)$ is path connected and clearly contains $w$ and $-w$. Thus, there is some continuous $\gamma: [0,1] \to \bR^d \cut \tau^{-1}(0)$ with $\gamma(0) = w$ and $\gamma(1) = -w$. Now the function 
    \[
        [0, 1] \ni t \mapsto \sum_{i=1}^n e_i \sigma(\gamma(t)^\TT x_i)
    \]
    is also continuous, positive at $0$ and negative at $1$, whence by the Mean Value theorem for continuous functions, there is some $t \in [0,1]$ with $\sum_{i=1}^n e_i \sigma(\gamma(t)^\TT x_i) = 0$, contradicting $\gamma(t) \notin \tau^{-1}(0)$. 

    Therefore, we conclude that $\tau^{-1}(0)$ is an analytic set of dimension at least $d-1$. Since $\tau$ is not constant-zero, it must be of dimension $d-1$, completing the proof. 
\end{proof}

\begin{thm}[Theorem \ref{Thm Connectivity of branches}]\label{AThm Connectivity of branches}
   The following results hold for branches of $\critR$ defined as in Definition \ref{Defn Branches of critR}. 
    \begin{itemize}
        \item [(a)] Given effective feature number $r_1$ and a partition $P_1 = (t_0, ..., t_{r_1})$ with $t_{r_1} + l_1 = m$. Let $s_1$ count the number of elements in the set $\{j: t_j - t_{j-1} > 0\}$. If the ineffective feature numbers $l_1, l_2$ satisfy $l_1 < l_2$, $r_1 + l_1 < m$ and $r_2 + l_2 \le m$, then there are partitions $P_{2,1}, ..., P_{2,s_1}$ and permutations $\pi_{2,1}, ..., \pi_{2, s_1}$ such that $\overline{\calC_{P_1}^{r_1,l_1}} \cap \calC_{P_{2,j}, \pi_{2,j}}^{r_1, l_2} \ne \emptyset$ for all $1 \le j \le s_1$. 

        \item [(b)] Suppose the effective feature numbers $r_1, r_2$ with $r_1 < r_2$. Given a partition $P_1$ as in (a), then there are ineffective feature numbers $l_1, l_2$, a partition $P_2$ and a permutation $\pi_2$ giving $\calC_{P_1}^{r_1,l_1} \cap \overline{\calC_{P_2, \pi_2}^{r_2,l_2}} \ne \emptyset$, if the following requirements are satisfied. 
        \begin{itemize}
            \item [i)]    $l_1 > 0$. 
            \item [ii)]   $2r_2 - r_1 + l \le m$. 
            \item [iii)]  For some $\theta \in \calC_{P_1}^{r_1,l_1}$ the common zero set of the $(d+1)$-many maps 
            \begin{align*}
                &w \mapsto \sum_{i=1}^n e_i(\theta) \sigma(w^\TT x_i) \\ 
                &w \mapsto \sum_{i=1}^n e_i(\theta) \sigma'(w^\TT x_i) (x_i)_1 \\ 
                &\,\, \vdots \\
                &w \mapsto \sum_{i=1}^n e_i(\theta) \sigma'(w^\TT x_i) (x_i)_d 
            \end{align*}
            has at least $r_2$-many elements. 
        \end{itemize}
    \end{itemize}
\end{thm}
\begin{proof}
\begin{itemize}
    \item [(a)] Since $r_1 + l_1 < m$, there must be some $j \in \{1, ..., r\}$ with $t_j - t_{j-1} > 0$. If $j \ne r$, we simply find a permutation $\pi_1 \in S_m$ that ``switches" the entries $a_{t_{j-1}+1}, w_{t_{j-1}+1}, ..., a_{t_j}, w_{t_j}$ and $a_{t_{r-1}+1}, w_{a_{r-1}+1}, ..., a_{t_r}, w_{t_r}$. Since $\overline{\calC_{P,\pi_1}^{r_1, l_1}} \cap \calC_{P_2, \pi_2}^{r_2, l_2}$ and $\calC_{P_1}^{r_1, l_1} \cap \calC_{P_2, \pi_1^{-1}\circ \pi_2}^{r_2, l_2}$ have the same geometry for any effective feature number $r_2$, any ineffective feature number $l_2$, and any partition $P_2$ and permutation $\pi_2$, we work instead with $\calC_{P_1, \pi_1}^{r_1, l_1}$. Thus, without loss of generality we may assume that $j = r$ in the beginning. \\

    Fix any $\theta = (a_k, w_k)_{k=1}^m \in \calC_P^{r_1, l_1}$. Denote $a' := \sum_{k=t_{r-1}+1}^{t_r} a_k$, which is non-zero by requirement (b) in Definition \ref{Defn Branches of critR}. Let $\theta' = (a_k', w_k')_{k=1}^m \in \bR^{(d+1)m}$ be such that 
    \begin{itemize}
        \item $a_k' = a_k$ for all $1 \le k \le t_{r-1}$ and all $k > t_r$. 
        \item $a_k' = 0$ for all $t_{r-1}+1 < k \le t_r$ and $a_{t_{r-1}+1}' = a'$. 
        \item $w_k' = w_k$ for all $1 \le k \le m$. 
    \end{itemize}
    Similarly, define for each $N \in \bN$ a point $\theta_N = (a_{k,N}, w_{k,N})_{k=1}^m \in \bR^{(d+1)m}$ be such that 
    \begin{itemize}
        \item $a_{k,N} = a_k$ for all $1 \le k \le t_{r-1}$ and all $k > t_r$. 
        \item $a_{k,N} = \frac{1}{2(t_r - t_{r-1}+1) N}$ for all $t_{r-1}+1 < k \le t_r$ and $a_{{t_{r-1}+1},N} = (1 - \frac{1}{2n}) a'$. 
        \item $w_{k,N} = w_k$ for all $1 \le k \le m$. 
    \end{itemize}
    Then clearly $\theta' \in \calC_{P_2}^{r_1,l_2}$ where $P_2 = (t_0, t_1, ..., t_{r-1}, t_{r-1}+1)$ and $l_2 = l_1 + (t_r - t_{r-1}-1)$. Similarly, for each $N \in \bN$ we clearly have $\theta_N \in \calC_{P_1}^{r_1, l_1}$. Meanwhile, we have $\theta' \in \overline{\calC_{P_1}^{r_1, l_1}}$ because 
    \[
        \theta' = \lim_{N\to\infty} \theta_N. 
    \]
    Therefore, $\theta' \in \overline{\calC_{P_1}^{r_1, l_1}}$, completing the proof. 

    \item [(b)] We will construct the branch $\calC_{P_2, \pi_2}^{r_2, l_2}$ explicitly. For simplicity, denote $P_1 = (t_0, t_1, ..., t_{r_1})$. Since $\calC_{P_2, \pi_2}^{r_2, l_2} \ne \emptyset$, $r_2 \le m - l_2$ must hold, so we may let $t_{r_2} := m - l_2$. Since $r_1 + 2(r_2 - r_1) + l \le m$, we can define a partition 
    \[
        P_2 = (t_0, ..., t_{r_1}, t_{r_1}+2, ..., t_{r_1} + 2(r_2 - r_1 - 1), t_{r_2}). 
    \]
     For permutation $\pi_2$, simply define $\pi_2$ as the identity map on $\{1, ..., m\}$ (in general there are more than one ways to define $\pi_2$). Because the given function have at least $r_2$-many common zeros, we can find some $\theta = (a_k, w_k)_{k=1}^m \in \calC_{P_1}^{r_1, l_1}$ such that $w_{t_{r_1}+2}, ..., w_{t_{r_1} + 2(r_2 - r_1 - 1)}, w_{t_{r_2}}$ are distinct, and moreover are also distinct from $w_1, ..., w_{t_{r_1}}$. For any $N \in \bN$, the $\theta_N = (a_{k,N}, w_{k,N})_{k=1}^m$ is a point in $\calC_{P_2, \pi_2}^{r_2, l_2}$, where 
     \begin{itemize}
         \item $a_{k,N} = a_k$ for all $1 \le k \le t_{r_1}$. 
         \item $a_{k,N} = \frac{1}{N}$ for all $k = t_{r_1} + 1, t_{r_1} + 3, ..., t_{r_1} + 2(r_2 - r_1 - 1) - 1$, and $a_{k,N} = -\frac{1}{N}$ for all $k = t_{r_1} + 2, t_{r_1} + 4, ..., t_{r_1} + 2(r_2 - r_1 - 1)$. 
         \item $a_{k,N} \ne 0$ for all $t_{r_1} + 2(r_2 - r_1 - 1) < t_{r_2}$, and they sum to zero. 
         \item $w_{k,N} = w_k$ for all $1 \le k \le m$. 
     \end{itemize}
     Notice that $\lim_{N\to\infty} \theta_N = \theta$. Thus, $\calC_{P_1}^{r_1, l_1} = \overline{\calC_{P_2, \pi_2}^{r_2, l_2}} \ne \emptyset$. 
\end{itemize}
\end{proof}

\begin{prop}[Proposition \ref{Prop Embedding saddles}]\label{AProp Embedding saddles}
    Let $r < m$. Suppose that $\theta' = (a_k', w_k')_{k=1}^r \in \calC_{P^*}^{r,0}(r)$ is a critical point of $R$ and the matrices
    \[
        \sum_{i=1}^n e_i(\theta') \sigma''(w_j'^{\TT} x_i) 
        \maThree{(x_i)_1(x_i)_1}{...}{(x_i)_1(x_i)_d}{\vdots}{\ddots}{\vdots}{(x_i)_d(x_i)_1}{...}{(x_i)_d(x_i)_d} 
    \]
    and $\sum_{i=1}^n \sigma(w_l'^{\TT} x_i)^2$ are non-zero for some $j$ and $l$. Then there is a partition $P = (t_0, ..., t_r = m)$ such that $\calC_P^{r,0}(m)$ contains strict saddle points $\theta = (a_k, w_k)_{k=1}^m$ satisfying $g(\theta, \cdot) = g(\theta', \cdot)$. 
\end{prop}
\begin{proof}
    Our proof is based on the idea of K. Fukumizu \citep{KFukumizu}. Note that $\nabla_w^2 R(\theta)$ is the sum of the following matrices 
    \[
        \begin{pmatrix}
            \maThree{a_1^2A_1}{...}{a_1a_{t_1}A_1}{\vdots}{\ddots}{\vdots}{a_{t_1}a_1 A_1}{...}{a_{t_1}^2 A_1} &... &* \\ 
            \vdots &\ddots &\vdots \\
            * &... &\maThree{a_{t_{r-1}+1}^2 A_r}{...}{a_{t_{r-1}+1}a_{t_r} A_r}{\vdots}{\ddots}{\vdots}{a_{t_r}a_{t_{r-1}+1} A_r}{...}{a_{t_r}^2 A_r} 
        \end{pmatrix}
    \]
    and 
    \[
    \begin{pmatrix}
        \maThree{a_1B_1}{...}{O}{\vdots}{\ddots}{\vdots}{O}{...}{a_{t_1}B_1} &... &O \\ 
        \vdots &\ddots &\vdots \\
        O &... &\maThree{a_{t_{r-1}+1} B_r}{...}{O}{\vdots}{\ddots}{\vdots}{O}{...}{a_{t_r} B_r}. 
    \end{pmatrix}
    \]
    In the two matrices above, $A_1,..., A_r$ are $d\times d$ matrices determined by $w_1', ..., w_r'$, $O$ are zero matrices (possibly with different sizes), and for each $1 \le j \le r$, 
    \[
        B_j = \sum_{i=1}^n \left( \sum_{k=1}^r a_k'\sigma(w_k'^{\TT} x_i) - y_i \right) \sigma''(w_j'^{\TT} x_i) \maThree{(x_i)_1(x_i)_1}{...}{(x_i)_1(x_i)_d}{\vdots}{\ddots}{\vdots}{(x_i)_d(x_i)_1}{...}{(x_i)_d(x_i)_d}. 
    \]
    Consider the vector $v = (va_k, vw_k)_{k=1}^m \in \bR^{(d+1)m}$ with $va_k = 0$ for all $1 \le k \le m$ and $vw_k = 0$ whenever $k \notin \{t_{j-1}+1, ..., t_j\}$ (the $va_k \in \bR$ and $vw_k \in \bR^d$ are entries of $v$). Then $\mathrm{Hess} R(\theta) v$ equals 
    \[
        \left[ \colThree{\maThree{a_{t_{j-1}+1}^2 A_j}{...}{a_{t_{j-1}+1}a_{t_j} A_j}{\vdots}{\ddots}{\vdots}{a_{t_j}a_{t_{j-1}+1} A_j}{...}{a_{t_j}^2 A_j}}{\vdots}{*} + \colThree{\maThree{a_{t_{j-1}+1} B_j}{...}{O}{\vdots}{\ddots}{\vdots}{O}{...}{a_{t_j}B_j}}{\vdots}{O} \right] \colThree{vw_{t_{j-1}+1}}{\vdots}{vw_{t_j}}. 
    \]
    Then $v^\TT \left(\mathrm{Hess} R(\theta) v\right)$ equals 
    \begin{align*}
        &\sum_{k,k'=t_{j-1}+1}^{t_j} a_k a_{k'} \left( (vw_k)^\TT A_j (vw_{k'}) \right) + \sum_{k=t_{j-1}+1}^{t_j} a_k \left( (vw_k)^\TT B_j (vw_k) \right) \\ 
        =\, &\left| \sum_{k=t_{j-1}+1}^{t_j} a_k \sqrt{A_j} vw_k \right|^2 + \sum_{k=t_{j-1}+1}^{t_j} a_k \left( (vw_k)^\TT B_j (vw_k) \right), 
    \end{align*}
    where $\sqrt{A_j}$ is the square root matrix of $A_j$, i.e., $\sqrt{A_j}\sqrt{A_j} = A_j$. By hypothesis, we may further assume that $B_j$ is not a zero matrix. Since $B_j$ is symmetric, there is some $u \in \bR^d$ with $u^\TT B_j u \ne 0$. Further consider any vector $v$ defined as above which also satisfies 
    \[
        v_{t_{j-1}+1} = ... = v_{t_j-1} = 0, v_{t_j} = u. 
    \]
    Then 
    \[
        v^\TT \left(\mathrm{Hess} R(\theta) v\right) = |\sqrt{A_j} u|^2 a_{t_j}^2 + \left(u^\TT B_j u\right) a_{t_j}. 
    \]
    Notice that we can view the right side of this equation as a polynomial in $a_{t_j}$ of degree at most 2. If $\sqrt{A_j}u = 0 \in \bR^d$, the for any $a_{t_j} < 0$ we have $v^\TT \left(\mathrm{Hess} R(\theta) v\right) < 0$. If $\sqrt{A_j} u \ne 0$, then for any $a_{t_j}$ between 0 and $- \frac{u^\TT B_j u}{|\sqrt{A_j}u|^2}$ we have $v^\TT \left(\mathrm{Hess} R(\theta) v\right) < 0$. In either case, we can see that $\mathrm{Hess} R(\theta)$ has a negative eigenvalue. \\
    
    Then we find a positive eigenvalue of $\mathrm{Hess} R(\theta)$. Indeed, our hypothesis implies that 
    \[
        \frac{\partial^2 R}{\partial a_l \partial a_l}(\theta) = \sum_{i=1}^n \sigma(w_l'^{\TT} x_i)^2 > 0.  
    \]
    Let $v \in \bR^{(d+1)m}$ be the vector which has 1 in the $a_{t_l}$-th entry and zero in all other entries. Then it is clear that 
    \[
        v^\TT \mathrm{Hess} R(\theta) v = |v|^2 \frac{\partial^2 R}{\partial a_l \partial a_l} > 0. 
    \]
    Thus, there must be a positive eigenvalue for $\mathrm{Hess} R(\theta)$. Since $\mathrm{Hess} R(\theta)$ has both negative and positive eigenvalues, $\theta$ is a strict saddle. 
    \\
\end{proof}

We then show use Proposition \ref{Prop Embedding saddles} to show that there exist strict saddles in $\calC^{1,0}$ in the example (see also Section \ref{Subsection Illust of main results}). 

\begin{lemma}\label{ALem for example}
    For the example in Section \ref{Subsection Illust of main results}, $\theta' = \left(1, \left(\log \frac{1}{3}, 0\right) \right)$ is a critical point of $R(g_1, \cdot)$ such that the matrices 
    \[
        \sum_{i=1}^4 e_i(\theta') \sigma''(w'^{\TT} x_i) \maTwo{(x_i)_1(x_i)_1}{(x_i)_1(x_i)_2}{(x_i)_2(x_i)_1}{(x_i)_2(x_i)_2}
    \]
    and $\sum_{i=1}^4 \sigma(w'^{\TT}x_i)^2$ are non-zero. Here $\sigma = \exp(\cdot)$.   
\end{lemma}
\begin{proof}
    We start by showing that $\theta'$ is a critical point of $R(g_1, \cdot)$. We compute 
    \begin{align*}
        e_1(\theta') &= 1 \cdot e^{1 \cdot \log 1/3} - 1 = -\frac{2}{3} \\ 
        e_2(\theta') &= 1 \cdot e^{0 \cdot \log 1/3} - 1 = 0 \\ 
        e_3(\theta') &= 1 \cdot e^{1 \cdot \log 1/3} - 0 = \frac{1}{3} \\ 
        e_4(\theta') &= 1 \cdot e^{1 \cdot \log 1/3} - 0 = \frac{1}{3}. 
    \end{align*}
    Therefore, the partial derivatives of $R(g_1, \cdot)$ at $\theta'$ are all zero, because 
    \begin{align*}
        \frac{1}{2} \parf{R}{a}(\theta') &= -\frac{2}{3} e^{1 \cdot \log 1/3} + 0 + \frac{1}{3} e^{1 \cdot \log 1/3} + \frac{1}{3} e^{1 \cdot \log 1/3} = 0,\\ 
        \frac{1}{2} \parf{R}{(w)_1}(\theta') &= -\frac{2}{3} e^{1 \cdot \log 1/3} + 0 + \frac{1}{3} e^{1 \cdot \log 1/3} + \frac{1}{3} e^{1 \cdot \log 1/3} = 0, \\ 
        \frac{1}{2} \parf{R}{(w)_2}(\theta') &= 0 + 0 + \frac{1}{3} e^{1 \cdot \log 1/3} \cdot 1 + e^{1 \cdot \log 1/3} \cdot (-1) = 0. 
    \end{align*}
    For the first matrix, it suffices to show that one of the entries of this $2\times 2$ matrix is non-zero. But this follows from 
    \begin{align*}
        \sum_{i=1}^4 e_i(\theta') \sigma''(w'^{\TT} x_i) (x_i)_2^2 
        &= 0 + 0 + \frac{1}{3} e^{1 \cdot \log 1/3} \cdot 1^2 + \frac{1}{3} e^{1 \cdot \log 1/3} \cdot (-1)^2 \\
        &= \frac{2}{9} \ne 0, 
    \end{align*}
    i.e., the entry at the second row and second column is non-zero. For the second matrix, note that when $\sigma = \exp(\cdot)$, $\sigma(w'^{\TT} x_i) > 0$ for all $1 \le i \le 4$, whence $\sum_{i=1}^4 \sigma(w'^{\TT} x_i) > 0$ as well. 
\end{proof}

Thus, the requirements in Proposition \ref{Prop Embedding saddles} are satisfied and we conclude that $\calC^{1,0}$ contains strict saddles. 

\begin{prop}[Proposition \ref{Prop Saddle branches}]\label{AProp Saddle branches}
    Let $r < m$. Any branch $\calC^{r,l}$ with $l > 0$ consists only of saddle points. Moreover, for any $\theta_0 \in \calC^{r,0}$, $\theta_0$ is connected to a saddle $\theta_1$ via a line segment $\gamma: [0,1] \to \bR^{(d+1)m}$, such that $\gamma(t) \in \critR$ and $g(\gamma(t), \cdot) = g(\theta_0, \cdot)$ for all $t$. 
\end{prop}
\begin{proof}
    First we show that any branch $\calC_P^{r,l}$ with $l > 0$ consists only of saddle points. So fix an arbitrary point $\theta = (a_k, w_k)_{k=1}^m \in \calC^{r,l}$. Since $l > 0$, there is some $j \in \{1, ..., m\}$ with $a_j = 0$. By the proof of Theorem \ref{AThm Embedding Geometry}, $w_j$ must lie in an analytic set of dimension $\le d-1$ in $\bR^d$. Thus, for any $\vep > 0$ there is some $\tw_j \in B(w_k, \vep) \siq \bR^d$ not in this set, and thus the point $\ttheta = (\ta_k, \tw_k)_{k=1}^m$ with $\ta_k = a_k$ for all $k$ and $\tw_k = w_k$ for all $k \ne j$ satisfies $R(\ttheta) = R(\theta)$ but is not a critical point, i.e., $\nabla R(\ttheta) \ne 0$. It follows that we can find $\ttheta_1, \ttheta_2 \in B(\ttheta, \vep)$ with 
    \[
        R(\ttheta_1) > R(\ttheta), \quad R(\ttheta_2) < R(\ttheta). 
    \]
    Notice that $\ttheta_1, \ttheta_2$ are both points in the open ball $B(\theta, 2\vep)$ because 
    \begin{align*}
        |\ttheta_i - \theta| 
        &\le |\ttheta_i - \ttheta| + |\ttheta - \theta| \\
        &= |\ttheta_i - \theta| + |\tw_j - w_j| < \vep + \vep. 
    \end{align*}
    Moreover, using $R(\ttheta) = R(\theta)$ we have $R(\ttheta_1) > R(\theta)$ and $R(\ttheta_2) < R(\theta)$. Since $\vep > 0$ is arbitrary, we can see that $\theta$ is neither a local minimum nor local maximum, whence a saddle. \\

    Now we show that any point $\theta_0$ in $\calC^{r,0}$ is connected to some $\theta_1 \in \calC^{r,1}$ via a line segment. Given $\theta_0 = (a_{k0}, w_{k0})_{k=1}^m \in \calC^{r,0}$, by rearranging the indices of it we may assume that $\theta_0 \in \calC_P^{r,0}$ for some partition $P = (t_0, t_1, ..., t_r)$ of $\{1, ..., m\}$. Since $r < m$, there is some $j \in \{1, ..., r\}$ with $t_j - t_{j-1} > 0$. Define $\theta_1 = (a_{k1}, w_{k1})_{k=1}^m$ as follows
    \begin{itemize}
        \item [(a)] $a_{k1} = a_{k0}$ for all $k \notin \{t_{j-1}+1, ..., t_j\}$. 
        \item [(b)] $a_{k1} = 0$ for all $t_{j-1} < k < t_j$ and $a_{t_j 1} = \sum_{k=t_{j-1}+1}^{t_j} a_{k0}$. 
        \item [(c)] $w_{k1} = w_{k0}$ for all $k$. 
    \end{itemize}
    Then define $\gamma(t) = (1-t) \theta_0 + t \theta_1$ for $t \in [0,1]$, i.e., $\gamma$ is the line segment connecting $\theta_0$ and $\theta_1$. Then it is easy to check that for each $t \in [0,1]$, $\gamma(t) \in \critR$ and $g(\gamma(t), \cdot) = g(\theta_0, \cdot)$. \\
\end{proof}
\begin{remark}
    In general such line segment $\gamma$ is not unique, so in general $\theta_0$ is connected to more than one saddle via line segments in $\critR$. 
\end{remark}

\end{document}